\newcommand\fedname{FedSense}

\documentclass[10pt,twocolumn,letterpaper]{article}

\usepackage[pagenumbers]{iccv} % To force page numbers, e.g. for an arXiv version

\usepackage[dvipsnames]{xcolor}

\definecolor{iccvblue}{rgb}{0.21,0.49,0.74}
\usepackage[pagebackref,breaklinks,colorlinks,allcolors=iccvblue]{hyperref}

\usepackage[linesnumbered,lined,ruled]{algorithm2e}
\usepackage[export]{adjustbox}
\usepackage{threeparttable}
\usepackage{multirow}
\usepackage{enumitem}
\usepackage{newfloat}
\usepackage{listings}
\usepackage{colortbl}
\usepackage{arydshln}
\usepackage{amsfonts}
\usepackage{caption}
\usepackage{wrapfig}
\usepackage{amssymb}
\usepackage{pifont}
\usepackage{float}
\usepackage{bbm}
\usepackage{bm}
\usepackage{mathrsfs}  
\usepackage{tikz}

\definecolor{mygray}{gray}{.9}
\definecolor{mygreen}{RGB}{93,173,85}
\definecolor{mywarning}{RGB}{233,144,61}

\definecolor{DarkBlue}{RGB}{64,101,149}
\definecolor{azure}{rgb}{0.0, 0.5, 1.0}
\definecolor{gray}{rgb}{0.3, 0.3, 0.3}
\definecolor{DarkGreen}{rgb}{0,0.5,0}
\definecolor{DarkYellow}{RGB}{191,144,0}
\definecolor{myyellow}{RGB}{237,209,68}
\definecolor{DarkRed}{rgb}{0.7, 0, 0} 
\definecolor{DarkBlue}{rgb}{0, 0, 0.7} 

\definecolor{lightblue}{HTML}{D7F6FF}

\newcommand{\circlednum}[2][black]{\strut\raisebox{-1pt}{\tikz{\node[circle,inner sep=0.6pt,fill=#1,font=\scriptsize\bfseries\color{white}]{#2};}}}

\newcolumntype{x}[1]{>{\centering\arraybackslash}p{#1pt}}
\newcolumntype{I}{!{\vrule width 1pt}}

\usepackage{tcolorbox}
\tcbuselibrary{theorems}
\definecolor{lightgray}{gray}{.9}
\definecolor{deepgray}{gray}{.8}
\tcbset{highlight math/.append style={left=0mm,right=0mm,top=0mm,bottom=0mm, colframe=white}}

\makeatletter
\newcommand{\thickhline}{%
    \noalign {\ifnum 0=`}\fi \hrule height 0.6pt
    \futurelet \reserved@a \@xhline
}



\definecolor{darksalmon}{rgb}{0.91, 0.59, 0.48}
\definecolor{emerald}{rgb}{0.31, 0.78, 0.47}
\definecolor{green(pigment)}{rgb}{0.0, 0.65, 0.31}
\definecolor{amaranth}{rgb}{0.9, 0.17, 0.31}
\definecolor{iris}{rgb}{0.35, 0.31, 0.81}
\definecolor{uu}{rgb}{0.95, 0.51, 0.51}
\definecolor{spirodiscoball}{rgb}{0.06, 0.75, 0.99}
\usepackage{bbding}

\def\paperID{1535} % *** Enter the Paper ID here
\def\confName{ICCV}
\def\confYear{2025}

\title{Towards Privacy-preserved Pre-training of Remote Sensing Foundation Models with Federated Mutual-guidance Learning}

\author{Jieyi Tan$^{1}$, Chengwei Zhang$^{2}$, Bo Dang$^{1}$, Yansheng Li$^{1,}$\thanks{Corresponding author.} \\
$^{1}$Wuhan University \quad $^{2}$University of Cambridge \\
{\tt\small tanjieyi@whu.edu.cn} \quad {\tt\small yansheng.li@whu.edu.cn} \\
}

\newtheorem{theorem}{Theorem}
\newtheorem{assumption}{Assumption}
\newtheorem{lemma}{Lemma}
\newtheorem{proof}{Proof}

\begin{document}
\maketitle
\begin{abstract}
Traditional Remote Sensing Foundation models (RSFMs) are pre-trained with a data-centralized paradigm, through self-supervision on large-scale curated remote sensing data. For each institution, however, pre-training RSFMs with limited data in a standalone manner may lead to suboptimal performance, while aggregating remote sensing data from multiple institutions for centralized pre-training raises privacy concerns. Seeking for collaboration is a promising solution to resolve this dilemma, where multiple institutions can collaboratively train RSFMs without sharing private data. In this paper, we propose a novel privacy-preserved pre-training framework \textbf{(\fedname)}, which enables multiple institutions to collaboratively train RSFMs without sharing private data. However, it is a non-trivial task hindered by a vicious cycle, which results from model drift by remote sensing data heterogeneity and high communication overhead. To break this vicious cycle, we introduce Federated Mutual-guidance Learning. Specifically, we propose a Server-to-Clients Guidance (SCG) mechanism to guide clients updates towards global-flatness optimal solutions. Additionally, we propose a Clients-to-Server Guidance (CSG) mechanism to inject local knowledge into the server by low-bit communication. Extensive experiments on four downstream tasks demonstrate the effectiveness of our \fedname~in both full-precision and communication-reduced scenarios, showcasing remarkable communication efficiency and performance gains.
\end{abstract}
    
\section{Introduction}
\label{sec:intro}

\begin{figure}
    \centering
    \includegraphics[width=1\linewidth]{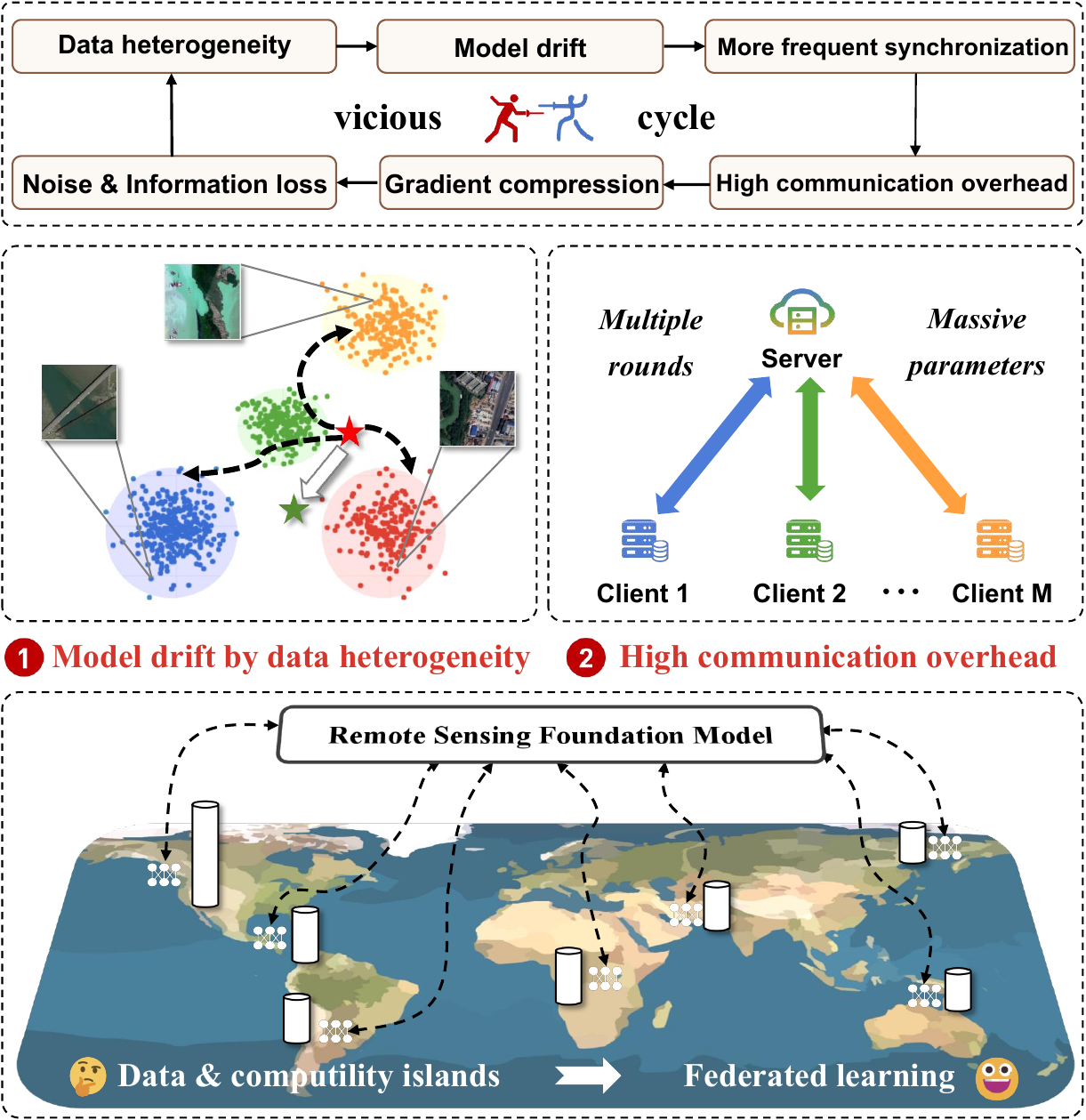}
    \caption{\textbf{Illustration of privacy-preserved pre-training of RSFMs with FL to bridge data islands.} The vicious cycle between data heterogeneity-induced model drift and communication bottlenecks reveals a critical performance-efficiency trade-off.
    }
    \label{fig:motivation}
\end{figure}

Recently, Remote Sensing Foundation Models (RSFMs) have gained increasing attention due to their impressive applicability and performance across various Earth Observation tasks by producing general-purpose visual features~\cite{rsfm,rse_dinov2}. Traditionally, existing RSFMs follow a data-centralized training paradigm. They are built through self-supervision on large-scale curated remote sensing data, gathered from diverse sources~\cite{skysense}. These data are collected by satellites, drones, or aerial platforms, and are stored in centralized archives by different institutions. A single institution could not pre-train RSFMs well due to the limited data scale and diversity~\cite{Innovation}. At the same time, it is challenging to aggregate data from multiple institutions for data-centralized training due to geo-information security, storage bottlenecks, and industrial competitions~\cite{intellect,FedLLM,GRSM_AI,GRSM_FL}. Such contradictory requirements urgently demand paradigm-shifting frameworks that reconcile computational synergies.

A more practical and realistic solution is to collaboratively inject remote sensing knowledge learned from private data owned by institutions into foundation models in a distributed manner~\cite{huayi,1,2,3,TPAMI_LAQ}. In recent years, federated learning (FL) emerges as a promising privacy-preserving alternative, enabling collaborative model training without raw data exchange through periodic model aggregation~\cite{FL_survey,nc_medical}. Self-supervised learning (SSL) operates on the principle of latent structure exploitation, where models learn transferable representations by solving pretext tasks derived from data's intrinsic attributes without human annotations. However, it is a non-trivial task to train RSFMs by combining SSL with FL directly~\cite{ssl_TPAMI}. Recent studies in federated self-supervised learning (FSSL) mainly focus on natural images, showing potential for distributed visual representation learning. MocoSFL~\cite{mocosfl} focus cross-device SSL by leveraging momentum contrast on mobile devices. $\text{FedU}^2$ designs disperses local data representations uniformly using spherical Gaussian sampling and optimizes global-local model consistency. FedMKD~\cite{FedMKD} proposes a resource-adaptive FSSL to address the architecture heterogeneity and class skew issues. However, few studies have explored the \textbf{vicious cycle} challenge in FSSL, which manifests more severely in remote sensing~\cite{a,b,c}. 

To elaborate, two critical challenges are involved in this vicious cycle, as illustrated in Fig.~\ref{fig:motivation}. \circlednum{1} \textbf{model drift by data heterogeneity.} Remote sensing data is inherently heterogeneous due to diverse sensor types, resolutions, and geographic distribution. This heterogeneity leads to significant variability in data distributions across different institutions, causing convergence inefficiencies and degraded model performance~\cite{GRSM_FL}. \circlednum{2} \textbf{high communication overhead.} Foundation models are characterized by massive parameters (unlike neural networks with shallow architectures), leading to excessive communication costs and bandwidth demands~\cite{swin}. The two challenges form a \textbf{vicious cycle}, where data heterogeneity necessitates more frequent model synchronization to mitigate drift, thereby amplifying communication costs. Conversely, communication compression can reduce overhead, introducing noise and information loss which further exacerbates client-side model drift. This bidirectional aggravation fundamentally undermines the efficiency and model consistency in distributed pre-training of RSFMs with FL.

In this paper, we propose a new challenging yet meaningful task: \textbf{privacy-preserved pre-training of RSFMs}. We propose \fedname, a novel FSSL framework with Federated Mutual-guidance Learning, to address the vicious cycle challenges. First, we introduce a Server-to-Client Mutual Guidance (SCG) mechanism to guide client models towards a global consensus, mitigating model drift by data heterogeneity. Second, we propose a Client-to-Server Guidance (CSG) mechanism to distill client models' knowledge into a server-side reference model, reducing communication overhead. Our \fedname~provides an integrated solution to resolve the vicious cycle challenges. As a result, our \fedname~achieves state-of-the-art performance compared to existing FSSL methods with higher performance and lower communication overhead. 

To sum up, this paper takes the first step towards privacy-preserved pre-training of RSFMs as far as we know. The main contributions of this work are as follows:
\begin{itemize}
    \item We propose \fedname, establishing a new paradigm for privacy-preserved pre-training of RSFMs. To the best of our knowledge, it is the first generic FL framework that supports mainstream pre-training methods, including contrastive learning and masked image modeling.
  
    \item We resolve the vicious cycle challenges in FSSL by introducing Federated Mutual-guidance Learning, which is composed of Server-to-Client Mutual Guidance (SCG) mechanism and Client-to-Server Guidance (CSG) mechanism. The integrated solution effectively mitigates model drift by data heterogeneity and reduces communication overhead.
   
    \item  We pre-train a RSFM with 10 participants with million-scale remote sensing data and evaluate the performance on four downstream tasks. Experimental results on eight datasets demonstrate that \fedname~outperforms existing FSSL methods in terms of performance and communication overhead.
\end{itemize}

\section{Related Work}
\label{related}

\begin{figure*}[t]
    \centering
    \includegraphics[width=1\linewidth]{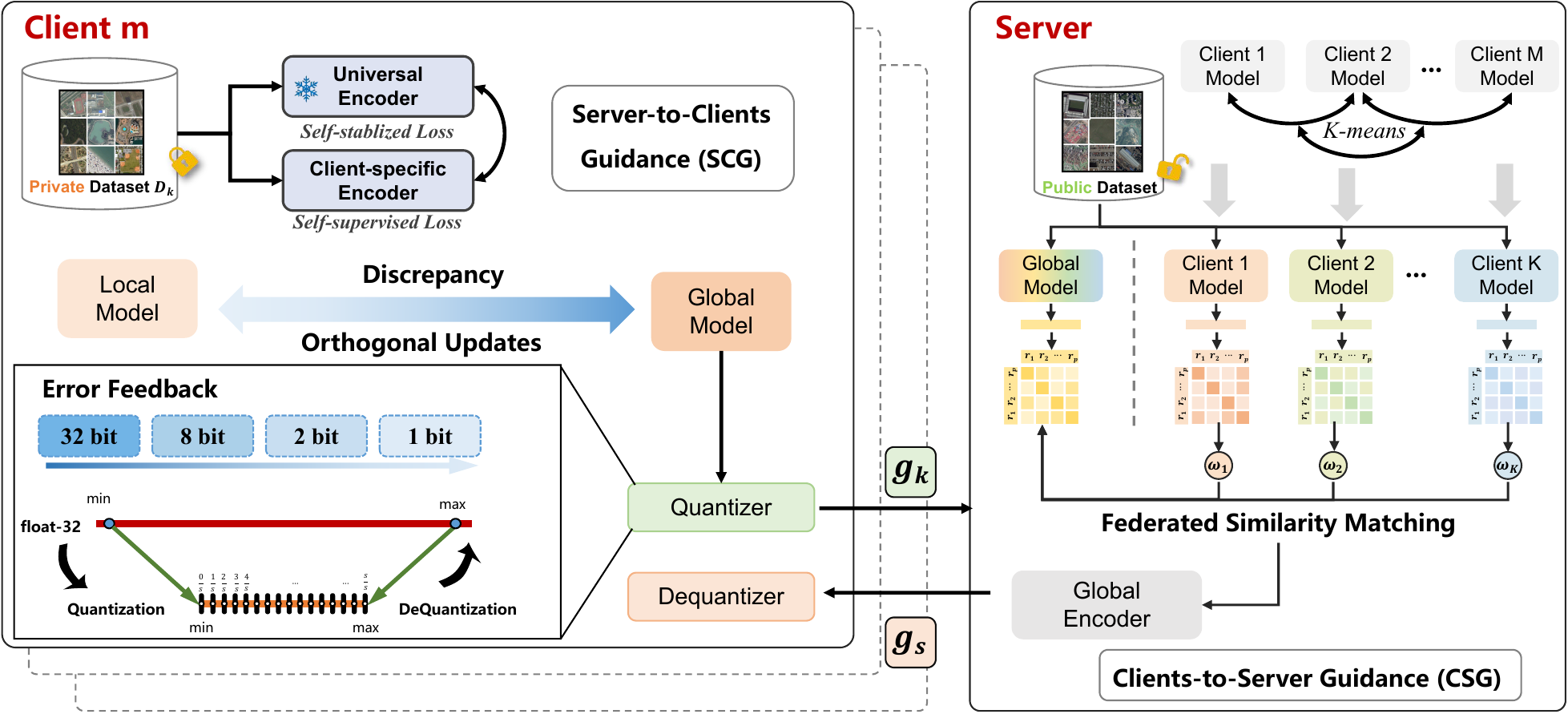}
    \caption{\textbf{Overview of \fedname.} The framework includes two components: Server-to-Clients Guidance (SCG) and Clients-to-Server Guidance (CSG). SCG guides clients' updates towards global-flatness optimal solutions, while CSG injects local knowledge into the server by low-bit communication.}
    \label{fig: method}
\end{figure*}

\subsection{Centralized Pre-training for RSFMs}
Recent years have witnessed significant progress in data-centralized pre-training paradigms for RSFMs~\cite{VFM_RRS,S2MAE}. A series of studies focus on constructing large-scale pre-training datasets and developing specialized algorithms tailored to remote sensing data characteristics. For instance, [9] introduced a rotation-variable window attention mechanism to handle large-scale and arbitrarily oriented geospatial objects, accompanied by MillionAID, a billion-parameter vision foundation model pre-trained on massive remote sensing imagery. To address dense object detection challenges in remote sensing, RingMo~\cite{ringmo} optimized a masked autoencoder framework by incorporating multi-scale hierarchical representations and task-specific decoding strategies. For multispectral data with rich spectral information, SpectralGPT~\cite{SpectralGPT} treated multispectral images as 3D tensors and proposed a multi-objective reconstruction loss to jointly capture spatial-spectral correlations and spectral sequence dependencies. SkySense~\cite{skysense} further extended multimodal capabilities through spatiotemporal disentanglement and temporal-aware embedding mechanisms, enabling joint contrastive learning across high-resolution optical imagery, and time-series optical/SAR data. This unified framework supports diverse tasks ranging from image-level classification to pixel-level segmentation and crop phenology monitoring. AnySat~\cite{anysat}is a versatile model designed to handle diverse data across resolutions, scales, and modalities.

Different from the above studies, our work takes the first step on privacy-preserved pre-training for RSFMs, and we aim to collaboratively pre-train RSFMs in a novel decentralized manner. This research is orthogonal to the existing centralized pre-training methods tailored to RSFMs, and would be complementary to scale up the performance of RSFMs in real-world applications.

\subsection{Federated Self-Supervised Learning}
Federated Learning (FL)~\cite{fedavg,fedsol} has emerged as a promising paradigm for collaborative model training across decentralized data sources while preserving data privacy. Federated self-supervised learning (FSSL)~\cite{SSL-FL,fedema,fedvssl} have demonstrated potential for privacy-preserving model training. However, existing FSSL methods predominantly address cross-device scenarios characterized by numerous resource-constrained clients (e.g., mobile devices) with homogeneous data distributions, focusing on computational efficiency and communication compression. While L-DAWA~\cite{L-DAWA} mitigates data heterogeneity through hierarchical angular divergence weighting, and $\text{FedU}^2$~\cite{FedU2} and FedMKD~\cite{FedMKD} enhance representation consistency via unified embedding alignment and adaptive knowledge distillation, these methods remain inadequate for cross-institution pre-training of RSFMs.

Existing FSSL focuses on optimizing federated model convergence and mainly support cross-device scenarios~\cite{q,w,e}. Additionally, these methods are with limited applicability to support self-supervised learning frameworks (contrastive learning and masked image modeling) consistently. Thirdly, these methods ignore the vicious cycle in collaboratively pre-training RSFMs, which requires handling the drift challenges and communication efficiency simultaneously. Our work sheds lights on these limitations and proposes Federated Mutual Learning to unleash the power of FSSL for RSFM pre-training.

\section{Methodology}
\label{method}

\subsection{Problem definition}
\label{sec:problem}
An FL framework for privacy-preserving collaborative pre-training of RSFMs aims to learn general-purpose visual representations from distributed unlabeled remote sensing images. The system comprises $M$ institutions (\ie clients) and a central server. Let the $m$-th {\color{DarkRed}{\textbf{client} $c_m$}} possess a private unlabeled dataset $\mathcal{D}_{c_m}$, and {\color{DarkGreen}{\textbf{server} $s$}} maintain a public unlabeled dataset $\mathcal{D}_s$. For each client $c_m$, the local self-supervised objective function $\mathcal{F}_{c_m}$ with parameters $\theta_m$ can be formulated as:
\begin{equation}
\min_{\theta_m} \mathcal{F}_{c_m} = \mathbb{E}_{x \sim \mathcal{D}_{c_m}} \left[ \psi\big(f_{\theta_m}(x^{(a)}), f_{\theta_m}(x^{(b)})\big) \right],
\end{equation}
where $x^{(a)}$ and $x^{(b)}$ are view augmentations or masked modeling of input image $x$, and $\psi(\cdot,\cdot)$ denotes the self-supervised loss function. For contrastive learning, $\psi$ measures feature consistency between augmented views. For masked modeling, $\psi$ computes the reconstruction error between predicted and original pixel values in masked regions. In the $t$-th round, upon receiving $\{\theta_m^{(t)}\}_{m=1}^M$, server $s$ with public data $\mathcal{D}_s$ applies aggregation $\Theta^{(t)} = \sum_{m=1}^M p_{m}\theta_m^{(t)}$, where $p_m=\frac{|\mathcal{D}_{c_m}|}{\sum_{m=1}^M |\mathcal{D}_{c_m}|}$. The optimization objective is:
\begin{equation}
\min_{\Theta} \mathcal{F}_s = \mathbb{E}_{x^s \sim \mathcal{D}_s} \Big[\phi(\Theta, \{\theta_m^{(t)}\}_{m=1}^M, \cdots, x^s)          \Big],
\end{equation}
where $x_s$ is unlabeled remote sensing data sampled from the public dataset, and $\phi$ denotes the loss function further to inject public data knowledge into the global foundation model.

\subsection{Overview of our \fedname}

As shown in~\cref{fig: method}, our \fedname~consists two parts, which are Server-to-Clients Guidance (SCG) and Clients-to-Server Guidance (CSG). The two components are introduced sequentially. We propose a federated self-supervised learning framework for collaborative training of foundation models across multiple institutions. Note that we omit the conversion between gradients and parameters during transmission for conciseness. The algorithm is detailed in~\cref{alg:FedSense}.

\subsection{Server-to-Clients Guidance}

The core insight of SCG is to strike a balance between global knowledge preservation and local model optimization. The server guides clients to optimize their local models while restrict the discrepancy of local and global model. We design a dual loss: self-supervised loss and self-stabilized loss in~\cref{eq:sst}, which are complementary to each other. The former requires models to learn orthogonality to the discrepancy of local and global model, while the latter helps to regularize the model with stabilized knowledge information in federated updates. Thus, the overall training target of the $m$-th client is:
\begin{equation}
\mathcal{L}_m^{\text{total}} = \underbrace{\mathcal{L}_m^{\text{ssl}}(\theta_m; \Theta;\epsilon; \mathcal{R}_C(\theta_m))}_{\text{self-supervised term}} + \underbrace{\mathcal{L}_m^{\text{sst}}(\theta_m;\theta_{uni})}_{\text{self-stablized term}},
\end{equation}
where ${\cal{L}}_m^{\rm{ssl}}$ denotes self-supervised loss. It corresponds to SSL objectives commonly used in the field, such as contrastive loss, and masked reconstruction loss. Note that this is a drift-aware loss incorporated with our proposed optimizing method. We conclude it as the
following minimax optimization problem:
\begin{equation}
\min_{\theta_m} \max_{|\epsilon|_2<\rho} \mathcal{L}_m^{\rm disc}(\mathbf{\theta}_m + \epsilon, f_{\theta_m}(x^{(a)}), f_{\theta_m}(x^{(b)})),
\end{equation}
The objective of the weight perturbation on the client-side is to find the $\epsilon$, which causes the maximum increase in the parameter discrepancy. Given local model $\theta_m$ and global model $\Theta$, the parameter discrepancy is calculated by:
\begin{equation}
   \nabla_{\theta_m} \mathcal{L}_m^{\rm{disc}}(\theta_m; \Theta)
     = \nabla_{\theta_m} (\beta (\theta_m - \Theta)), 
\end{equation}
where $\beta$ is a hyperparameter to of the discrepancy term. Then, the optimal perturbation $\epsilon$ is approximated by:
\begin{equation}
\widetilde{\epsilon} = \text{argmax}\mathcal{L}_m^{\text{disc}}(\theta_m + \epsilon; \Theta) \approx \lambda \frac{\nabla_{\theta_m} \mathcal{L}_m^{\rm{disc}}(\theta_m; \Theta)}{\|\nabla_{\theta_m} \mathcal{L}_m^{\rm{disc}}(\theta_m; \Theta)\|_2},
\end{equation}
where $\lambda$ is a scaling factor. The optimal perturbation $\epsilon$ will be later used to update the local model $\theta_m$.

Moreover, the self-stabilized loss $\mathcal{L}_m^{\text{sst}}$ is designed to leverage the knowledge from a universal encoder, $f_{\theta_{\rm{uni}}}$, which is pre-trained on a large-scale dataset (e.g., ImageNet-22k) and broadcast to all clients by the server at the beginning of the federated learning process. It guides the client-specific encoder to output representations that align with the universal encoder. Formally, the self-stabilized loss can be defined as:
\begin{equation}
\mathcal{L}_m^{\rm{sst}} = \mathbb{E}_{x \sim \mathcal{D}_m} \left[ -\frac{f_{\theta_{m}}(x)}{\|f_{\theta_{m}}(x)\|_2} \cdot \frac{f_{\rm{uni}}(x)}{\|f_{\rm{uni}}(x)\|_2} \right],
\label{eq:sst}
\end{equation}
where $\mathcal{D}_m$ denotes the local dataset of client $m$. $f_{\theta_{\rm{uni}}}$ and $f_{\theta_{m}}$ are the feature extractors of the universal encoder and the client-specific encoder, respectively.
Lastly, the optimal perturbation $\epsilon$ is then used to update the local model $\theta_m$: \begin{equation}
    \theta_m \leftarrow \theta_m - \gamma \nabla_{\theta_m} \mathcal{L}_m^{\rm{total}}(\theta_m + \epsilon; \Theta),
    \label{eq:local_update}
\end{equation}
where $\gamma$ is the learning rate. In this way, the server guides the clients to optimize their local models while preserving the global knowledge with orthogonal property and stabilized information.

\begin{algorithm}[t]
   \caption{Our \fedname}
   \label{alg:FedSense}
   \SetKwInOut{Input}{Input}
   \SetKwInOut{Output}{Output}
    \Input{$T$: communication round;\ \ $M$: number of clients; \ \ $E$: local epochs;} 
    \Output{Weight $\Theta^T$ of the RSFM at the $T$-th round.}

\underline{\color{DarkGreen}\textbf{Server-side:}}
    
    Initialize global model $\Theta^0$ \\
    \For(\tcp*[f]{the $t$-th round}){$t=1$ {\bfseries to} $T$}{
         \ \ Server broadcast $\Theta^{t-1}$ to selected clients \\      
        \For{each client $c_m\in \{c_m\}_{m=1}^{M}$ {\rm\bfseries in parallel}}{        
            Server-to-Clients Guidance (SCG)\\
           
        }

        $\{\Delta \theta_m^t\}_{m=1}^M =\text{\textit{DCPR}} \{\{\Delta \widetilde\theta_m^t\}_{m=1}^M;b\}$

        $\Theta^t \leftarrow \text{ServerAggregation}(\{\Delta\theta_m^t\}_{m=1}^M)$  \\
        Clients-to-Server Guidance (CSG)\\
        Similarity alignment with public data \cref{eq:distill} \\

    }
    
\underline{\textbf{\color{DarkRed}Client-side:}}

 Initialize local model
            $\theta_m^{t-1} \leftarrow \Theta^{t-1}$    \\
 \For(\tcp*[f]{the $e$-th epoch}){$e=1$ {\bfseries to} $E$}{
           
           Server-to-Clients Guidance (SCG)\\
           Optimization with knowledge preservation\\
           $ \theta_m \leftarrow \theta_m - \gamma \nabla_{\theta_m} \mathcal{L}_m^{\rm{total}}(\theta_m + \epsilon; \Theta)$    \cref{eq:local_update} \\
        
    }
    
     $\Delta \widetilde\theta_{m}^{t}$ = $\text{\textit{CPR}}\{(\theta_m^t-\theta_m^{t-1});b\}$
            \\ Send the updates  $\Delta \widetilde\theta_m^t$ back to the server.

\underline{\textbf{\color{DarkBlue}Downstream tasks:}} Institutions utilize the pre-trained model as a backbone, fine-tuning it on labeled data for specific tasks.

\end{algorithm}

\subsection{Clients-to-Server Guidance}
As the model size and number of clients increases, the communication overhead becomes a more serious bottleneck. The clients need to transmit the updates to the server in each communication round, and waiting for the server for aggregation and send back the updated parameters. 

For the uplink communication is the most time-consuming part in this process, we propose a CSG mechanism to reduce the communication cost. Though quantization is a widely used technique to reduce the communication cost. However, it inherently introduces quantization error and information loss, which may accumulate over time and degrade the model performance. To guide the server to aggregate the quantized updates from clients, we propose a feedback error mechanism to compensate the quantization error.

Assume the feedback error at the $t$-th round is $e_{m}^{t}$, it is added to compensate updates of each client:
\begin{equation}
    {{\cal{G}}_{m}^{t+1}}=\Delta\theta_{m}^{t}+e_{m}^{t},    
\end{equation}
where $\Delta\theta_{m}^{t}$ is the updates of client $m$ at round $t$. The quantized updates are computed as follows:
\begin{equation}
    {{\cal{\widetilde G}}_{m}^{t+1}}[i] = 
    \underbrace{||\boldsymbol{{\cal{G}}_{m}^{t+1}}||}_{\substack{\text{L2 norm of} \\ \text{raw updates}}} 
    \cdot 
    \underbrace{\text{sgn}({{\cal{G}}_{m}^{t+1}}[i])}_{\substack{\text{sign of element} \\ \text{(1 bit)}}} 
    \cdot 
    \underbrace{\xi({{\cal{G}}_{m}^{t+1}}[i]; s)}_{\substack{\text{unbiased stochastic} \\ \text{function(\(b-1\) bits)}}},
    \end{equation}
where $\xi(\cdot)$ is a unbiased stochastic function mapping $|{{\cal{G}}_{m}^{t+1}}[i]|/||\boldsymbol{{{\cal{G}}_{m}^{t+1}}}||$ to the quantization level $s$.
\begin{equation}
    \colorbox{lightgray}{$e_{m}^{t}$} = {{\cal{G}}_{m}^{t}} -\text{DCPR} \left( {{\cal{\widetilde G}}_{m}^{t}}; b \right),
\end{equation}
where DCPR is the dequantization function. The feedback error is updated by:
\begin{equation}
    e_{m}^{t} = \alpha \cdot e_{m}^{t-1} + (1-\alpha)\cdot \colorbox{lightgray}{$e_{m}^{t}$},
\end{equation}
where $\alpha$ is the momentum factor, and b is the bit-width of quantization. The feedback error is accumulated over time to preserve the gradient information across communication rounds. Inspired by the dynamic optimization characteristics of neural network training~\cite{loco}, we implement periodic feedback error resetting to address the fast-evolving loss landscape. We reset the feedback error to zero at the frequency of $T_{\text{reset}}$ rounds to ensure the feedback error remains aligned with the current optimization state.

On the server side, we propose federated similarity distillation to provide public remote sensing data guidance. Our \fedname~consists of three core steps: server-side aggregation, local model clustering, and cross-model knowledge distillation. In this way, clients can leverage the public data to enhance the global model performance. The server aggregates client models $\{\theta_m\}_{m=1}^M$ using data volume weights. They are clustered into $K$ groups via K-means to accelerate multi-model forward pass cost and mitigate the bias of some models:
\begin{equation}
\{\theta^{(k)}\}_{k=1}^K = \text{K-means}(\{\theta_m\}_{m=1}^M).
\end{equation}
For each public batch $\mathcal{B} = \{x_i\}_{i=1}^p$, the group of models produce feature matrices:
\begin{align}
\mathbf{Z}^{(k)} &= [z_1^{(k)}, \ldots, z_p^{(k)}]^\top \in \mathbb{R}^{p \times d}, \\
\mathbf{Z}^g &= [z_1^g, \ldots, z_p^g]^\top \in \mathbb{R}^{p \times d},
\end{align}
where $z_i^{(k)} = f_{\theta^{(k)}}(x_i)$ and $z_i^g = f_{\theta^g}(x_i)$, and $p$ and $d$ are the batch size and feature dimension, respectively. The similarity matrices are computed as:
\begin{equation}
\mathbf{S}^{(k)} = \frac{\mathbf{Z}^{(k)} (\mathbf{Z}^{(k)})^\top}{\|\mathbf{Z}^{(k)}\|_F}, \quad 
\mathbf{S}^g = \frac{\mathbf{Z}^g (\mathbf{Z}^g)^\top}{\|\mathbf{Z}^g\|_F}.
\end{equation}
The weighted consensus similarity combines local expertise:
\begin{equation}
\mathbf{S}_{\text{consensus}} = \sum_{k=1}^N \omega_k \mathbf{S}^{(k)}, \quad \omega_k = \frac{|\mathcal{D}_k|}{\sum_{i=1}^N |\mathcal{D}_i|}.
\end{equation}
The global model is optimized by matching similarity distributions:
\begin{equation}
\mathcal{L}_{\text{distill}} = \frac{1}{p^2} \|\mathbf{S}^g - \mathbf{S}_{\text{consensus}}\|_F^2.
\label{eq:distill}
\end{equation}
By distilling the similarity knowledge from multiple models, the global model can learn to capture the intrinsic structure of the public data, which is beneficial for enhancing the generalization ability of the global model.

\section{Experiments}
\label{experiment}

\subsection{Federated Experimental Setups}

\noindent \textbf{Distributed Pre-training datasets.} A distributed dataset was constructed for federated pre-training of RSFMs, comprising 10 clients with heterogeneous private remote sensing data and supplementary public datasets maintained by a server (\cref{fig:pretrain_data}). Notably, the dataset includes satellite images (\eg WorldView-2 and JL-1) and aerial images (\eg NAIP and NOAA), while some institutions possess multi-sourced collections, featuring heterogeneous sensor configurations and spatial resolutions ranging from 0.25m to 25m. The coverage spans diverse geographical regions including global, USA, and China. Such heterogeneity across data volume, geographic distribution, resolution variance, and platform diversity establishes a representative million-scale dataset that simulates real-world scenarios.

\begin{figure}[t]
    \centering
    \includegraphics[width=1\linewidth]{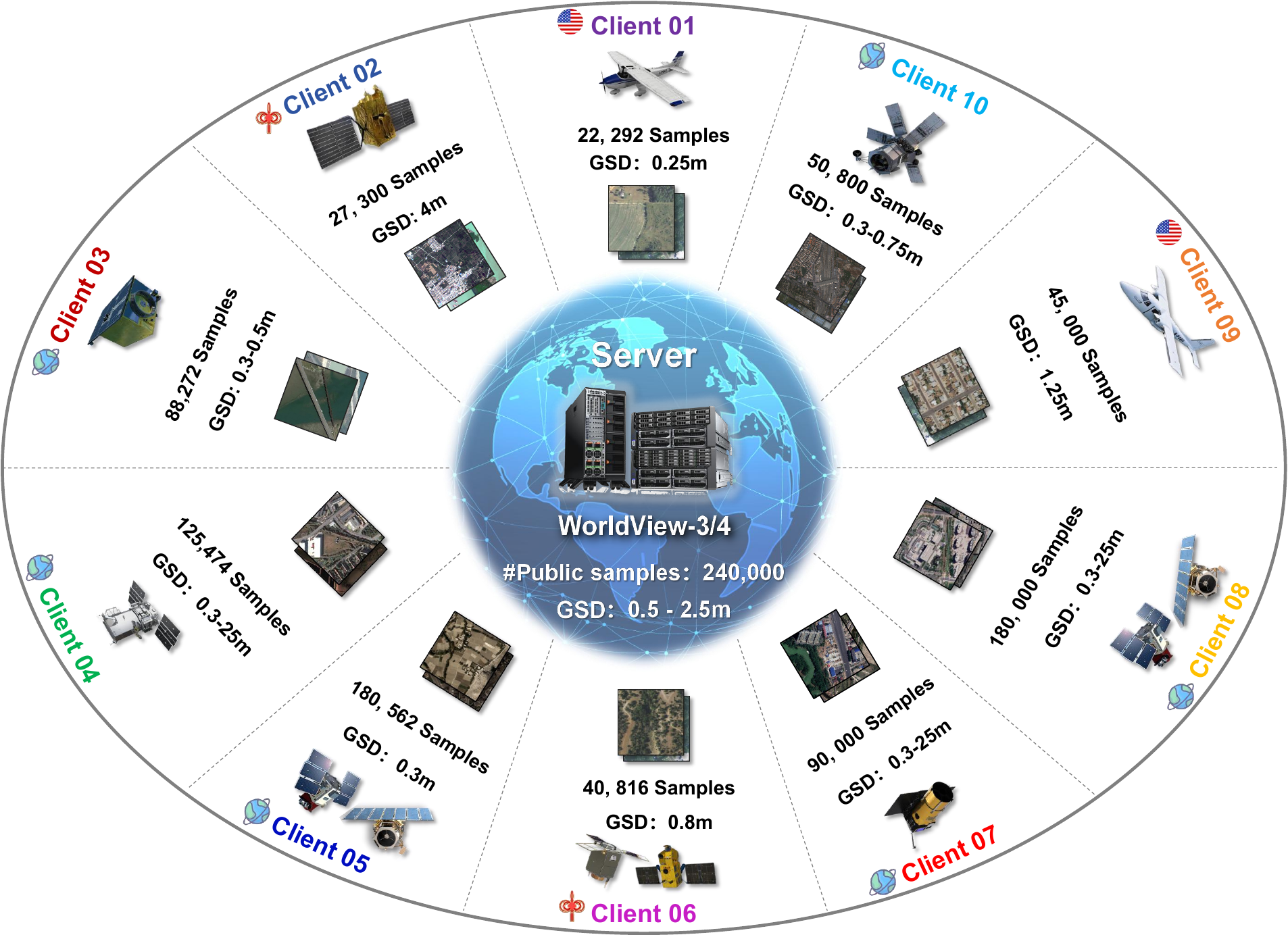}
    \caption{\textbf{Details of federated pre-training datasets.} The dataset consists of 10 clients with million-scale heterogeneous private remote sensing data and public datasets maintained by a server.}
    \label{fig:pretrain_data}
\end{figure}

\noindent \textbf{Downstream tasks.} We conducted experiments on four typical downstream tasks of remote sensing image analysis to validate the effectiveness of the RSFM pretrained by our \fedname. These tasks include: scene classification (RESISC-45~\cite{NWPU}, AID~\cite{AID}); semantic segmentation (Inria~\cite{Inria}, LoveDA~\cite{LoveDA}); object detection (DIOR-R~\cite{DIORR}, DOTA-v1.0~\cite{DOTA}); and change detection (LEVIR-CD+~\cite{LEVIRCD}, SECOND~\cite{SECOND}). More details of the datasets are shown in \cref{sec:datasets} of the supplementary.

\noindent \textbf{Evaluation metrics.}
For scene classification, we emply overall accuracy (OA) as the evaluation metric. We use mean intersection over union (mIoU) for multi-class semantic segmentation tasks and IoU for binary segmentation tasks. For object detection, we use mean Average Precision (mAP) and F1 score for change detection tasks. For semantic change detection, we use semantic change segmentation score (SCS) as the evaluation metric.

\noindent \textbf{Implementation details.} To systematically evaluate the universality of our framework, we establish a comprehensive evaluation framework incorporating two mainstream SSL paradigms: contrastive learning through DINO~\cite{dino}, and masked image modeling via SimMIM~\cite{simmim}. All experiments employ the tiny version of Swin Transformer (Swin-T) as the foundational backbone unless explicitly stated, with consistent training protocols (100-round pre-training, AdamW optimizer, 1e-4 base learning rate).

\subsection{Comparison with State-of-the-Art Methods}

\begin{table*}[t]\small
\centering
\resizebox{\textwidth}{!}{%
\setlength\tabcolsep{5pt}
\renewcommand\arraystretch{1.1}
\begin{tabular}{cccc|ccc|cc|cccc}
\hline \thickhline

\multirow{2}{*}{Framework}    &   \multirow{2}{*}{Method}   & RESISC-45 & AID     & \multicolumn{2}{c}{Inira}     & LoveDA  & DIOR-R  & DOTA-v1.0 & \multicolumn{3}{c}{LEVIR-CD+} & SECOND \\ \cline{3-13} 

 &   & OA & OA  & IoU  & OA  & mIoU  & mAP &mAP &  Precision & Recall & F1 & SCS              \\
 \hline 
\multirow{2}{*}{Swin-Tiny}                                                            
& Random Init.                     
& 80.09   & 72.84  & 79.41 & 90.21 & 42.83 & 46.86  & 64.45 & 69.71 & 63.39 & 66.40 & 30.95 \\
& ImageNet Sup.             
& 94.45  & 96.36 & 80.65 &93.16 & 51.52 & 64.37 & 76.90 & 73.26 & 71.44 & 72.34 &34.79 \\
 \hline
\multirow{7}{*}{\begin{tabular}[c]{@{}c@{}}SimMIM~\cite{simmim}\\ (Swin-Tiny)\end{tabular}} 
& SSL-FL~\cite{SSL-FL}      
& 95.21     & 96.17 & \underline{81.33} &93.43 & 51.67 & 64.01 & \underline{76.82} & \underline{74.31} & 70.99 & 72.61 &33.98       \\
& FedEMA~\cite{fedema}                         
& \multicolumn{2}{c|}{-}  & \multicolumn{3}{c|}{-} & \multicolumn{2}{c|}{-} & \multicolumn{4}{c}{-}       \\
& L-DAWA~\cite{L-DAWA}                       
& \multicolumn{2}{c|}{-}  & \multicolumn{3}{c|}{-} & \multicolumn{2}{c|}{-} & \multicolumn{4}{c}{-}       \\
& $\mathrm{FedU^{2}}$~\cite{FedU2}                        
& \multicolumn{2}{c|}{-}  & \multicolumn{3}{c|}{-} & \multicolumn{2}{c|}{-} & \multicolumn{4}{c}{-}       \\
& FedMKD~\cite{FedMKD}                        
& \multicolumn{2}{c|}{-}  & \multicolumn{3}{c|}{-} & \multicolumn{2}{c|}{-} & \multicolumn{4}{c}{-}       \\

 \rowcolor{Goldenrod!15} 
  \cellcolor{white}   & \cellcolor{white}    &
\textbf{96.33}  & \textbf{97.54}  & \textbf{81.66} &\textbf{94.28} & \textbf{52.74} & \textbf{65.51} & \textbf{77.33} & \textbf{74.82} & \textbf{71.67} & \textbf{73.21} &\textbf{35.43}\\  
\rowcolor{lightblue} 
 \cellcolor{white}    &  \cellcolor{white}  \multirow{-2}{*}{Our FedSense}    & 
\underline{96.01}  & \underline{96.68}  & 80.97 &\underline{93.45} & \underline{51.91} & \underline{64.53} & 76.76 & 73.98 & \underline{71.59} & \underline{72.77} &\underline{34.87}       \\
 \hline

\multirow{7}{*}{\begin{tabular}[c]{@{}c@{}}DINO~\cite{dino}\\ (Swin-Tiny)\end{tabular}} 
& SSL-FL~\cite{SSL-FL}      
& \multicolumn{2}{c|}{-}  & \multicolumn{3}{c|}{-} & \multicolumn{2}{c|}{-} & \multicolumn{4}{c}{-}       \\
& FedEMA~\cite{fedema}                            
& 94.91  & 97.04 & 80.21 &92.98 & 51.89 & 64.93 & 77.09 & 73.63 & 71.20 & 72.39 &34.82       
\\
& L-DAWA~\cite{L-DAWA}                          
& 94.23  & 96.27 & 80.82 &93.33 & 51.68 & 65.08 & 77.24 & \underline{74.18} & 71.82 & 72.98 &34.97 
\\
& $\mathrm{FedU^{2}}$~\cite{FedU2}                          
& 93.85  & 95.76 & 80.13 &92.59 & 51.34 & 64.03 & 75.88 & 73.30 & 70.67 & 71.96 &33.27        
\\
& FedMKD~\cite{FedMKD}                        
& 95.34  & 97.17 & \underline{80.84} &\underline{94.14} & \textbf{51.99} & \textbf{65.44} & \underline{77.56} & 74.36 & 73.05 & \underline{73.70} &35.11         \\
\rowcolor{Goldenrod!15} 
  \cellcolor{white}   & \cellcolor{white}    &
\textbf{96.21}    & \textbf{97.43} & \textbf{81.96} &\textbf{94.23} & \underline{51.95} & \underline{65.40} & \textbf{77.64} & \textbf{74.98} & \textbf{73.13} & \textbf{74.04} &\underline{35.34}         
\\ 
\rowcolor{lightblue} 
 \cellcolor{white}    &  \cellcolor{white}  \multirow{-2}{*}{Our FedSense}    & 
\underline{95.88}  & \underline{97.23}  & 80.97 &93.88 & 51.14 & 64.62 & 77.45 & 73.88 & \underline{73.09} & 73.48 &\textbf{35.82}       
\\
\hline \thickhline
\end{tabular}%
}
\caption{\textbf{Comparison results of our \fedname~and previous SOTA methods.} The symbol (-) indicates unavailable results where methods are incompatible with specific SSL framework types. Results highlighted in \textcolor{myyellow}{\textbf{yellow}} denote full-precision communication performance of our \fedname, while \textcolor{cyan}{\textbf{blue}} shading represents experiments with 1-bit communication-quantized transmission, demonstrating our method's efficiency-accuracy trade-off. The best results are highlighted in \textbf{bold}, and the second-best results are \underline{underlined}.}
\label{tab:total}
\end{table*}

In this section, we present a comparative analysis of performance on four downstream tasks using RSFM pre-trained by our \fedname~and other state-of-the-art FSSL. The results are summarized in \cref{tab:total}. We observe that our \fedname~outperforms existing methods across almost all tasks, achieving an average improvement of 1.3\% OA on scene classification, 1.1\% mIoU on semantic segmentation, 0.9\% mAP on object detection, and 0.8\% F1 on change detection. Note that due to not using public dataset in the original paper, we conduct experiments by averagely integrate public datasets to the federated pre-training dataset for a fair comparison with existing methods with the same number of total samples. To elaborate, we provide detailed comparisons with existing methods.

Our experimental results reveal several key observations. The randomly initialized model yields inferior performance across all downstream tasks, demonstrating Transformers' inherent limitation in learning effective representations from limited labeled data. ImageNet supervised pre-training brings significant improvements (\eg, +14.36\% on RESISC-45), validating the importance of pre-training for vision transformers. However, the domain gap between natural images and top-down remote sensing views limits further performance gains, motivating our FL approach.

Among FSSL methods, most approaches (FedEMA, FedMKD) achieve moderate improvements over ImageNet pretraining, confirming that collaborative pretraining with distributed data can inject domain-specific knowledge. Notably, $\text{FedU}^2$ exhibits performance degradation compared to ImageNet pretraining, suggesting its vulnerability to data heterogeneity and insufficient utilization of pre-trained knowledge. Our FedSense consistently outperforms SSL-FL, achieving absolute gains of 1.12\% and 1.37\% on RESISC-45 and AID for scene classification. When using DINO framework, our FedSense maintains advantages with 0.87\% and 0.26\% improvements, respectively.

For semantic segmentation, FedSense achieves 0.85\% OA improvement on Inria and 1.07\% gain on LoveDA over the best competitor. The marginal differences under DINO framework (0.12\%-0.35\%) suggest contrastive learning brings limited benefits for segmentation tasks that require precise pixel-level localization. In rotated object detection, FedSense with SimMIM framework surpasses SOTA by 1.50\% and 0.51\% on DIOR-R and DOTA-v1.0, while maintaining competitive performance (+0.11\% mAP) with DINO-based approaches.

Additionally, the 1-bit quantized version of FedSense demonstrates remarkable communication efficiency with minimal performance degradation (0.23\% across tasks). Quantized FedSense slightly outperforms full-precision baselines on SECOND dataset (+0.06\% SCS), which we conjecture stems from the error-compensated quantization acting as implicit regularization. This makes FedSense particularly suitable for bandwidth-constrained applications.

These results collectively demonstrate that our Federated Mutual-guidance Learning framework effectively coordinates distributed clients to learn transferable representations while maintaining communication efficiency. The consistent improvements across diverse tasks validate FedSense's ability to capture domain-specific patterns from unlabeled remote sensing data through federated collaboration.

\subsection{Ablation Studies}
In this part, we conduct ablation studies to analyze the system scalability, model scalability, effectiveness of components, and parameter analysis of our proposed~\fedname.

\textbf{System Scalability Analysis.}
To assess our framework's adaptability to real-world distributed scenarios, we systematically investigate how model performance scales with increasing participants and training samples. As shown in~\cref{tab:participants}, expanding from 2 to 10 collaborative clients (80K to 800K samples) yields consistent performance gains across all tasks. For RESISC-45 and DIOR-R, the performance improvements scale nearly linearly with client/sample quantities, suggesting these tasks particularly benefit from diverse perspectives in \fedname. However, LoveDA exhibits marginal gains despite 10× sample growth, indicating pixel-level tasks require more sophisticated feature aggregation beyond simple data scaling. These findings confirm our framework's effectiveness in harnessing distributed resources.

\begin{table}[htbp]
    \centering
    \resizebox{\columnwidth}{!}{%
    \setlength\tabcolsep{4pt}
    \renewcommand\arraystretch{1.6}
    \begin{tabular}{cc|cccc}
    \hline \thickhline
    \multirow{2}{*}{\#CL} & \multirow{2}{*}{\#TS} & RESISC-45 & LoveDA & DIOR-R & LEVIR-CD+ \\ \cline{3-6}
        &              & OA & mIoU & mAP & F1 \\
        \hline
    2   & 80K          & 94.50   & 51.56     & 64.43    &   72.44 \\
    4   & 200K         & 95.23   & 51.86     & 64.75    &   72.68 \\
    8   & 650K         & 95.85   & 52.33     & 65.16    &   73.00 \\
    10  & 800K         & \textbf{96.33}   & \textbf{52.74}     & \textbf{65.51}    &   \textbf{73.21} \\
    \hline \thickhline
    \end{tabular}%
    }
    \caption{\textbf{Number of participants impact analysis.} \#CL means the number of participants. \#TS means number of total participating samples.}
    \label{tab:participants}
    \end{table}

\textbf{Model Scalability Analysis.}
The impact of model capacity is systematically evaluated through progressively larger Swin Transformer variants, as detailed in~\cref{tab:model_size}. While all tasks benefit from increased model parameters, we observe distinct scaling patterns across task types. RESISC-45 shows diminishing returns, improving only +1.21\% from Swin-T to Swin-L, suggesting vision transformers approach saturation points for scene classification. Conversely, segmentation (LoveDA mIoU +3.22\%) and detection (DIOR-R mAP +3.31\%) exhibit near-linear improvements with model growth, indicating complex localization tasks inherently demand higher-capacity architectures. Notably, change detection (LEVIR-CD+ F1 +2.92\%) demonstrates sustained sensitivity to model size, likely requiring deeper feature hierarchies to discern subtle temporal changes. The 197M-parameter Swin-L achieves marginal gains (+0.42\% mAP over Swin-B) compared to its 2.2× parameter increase, highlighting practical trade-offs between model capacity and computational costs.

\begin{table}[htbp]
    \centering
    \resizebox{\columnwidth}{!}{%
    \setlength\tabcolsep{4pt}
    \renewcommand\arraystretch{1.6}
    \begin{tabular}{cc|cccc}
    \hline \thickhline
    \multirow{2}{*}{Model} & \multirow{2}{*}{\#Para.} & RESISC-45 & LoveDA & DIOR-R & LEVIR-CD+ \\ \cline{3-6}
        &              & OA & mIoU & mAP & F1 \\
        \hline
    Swin-T   & 28M       & 96.33   & 52.74     & 65.51    &   73.21   \\
    Swin-S   & 50M       & 96.80   & 54.12     & 67.03    &   74.65   \\
    Swin-B   & 88M       & 97.15   & 55.29     & 68.40    &   75.83   \\
    Swin-L   & 197M      & \textbf{97.54}   & \textbf{55.96}     & \textbf{68.82}    &   \textbf{76.13}   \\
    \hline \thickhline
    \end{tabular}%
    }
    \caption{\textbf{Model size impact analysis.} \# Para. means the number of model parameters.}
    \label{tab:model_size}
    \end{table}

\textbf{Quantization Methods Comparison.}
We compare our proposed quantization method with uniform quantization and FedPAQ~\cite{AISTATS_FedPAQ} on the RESISC-45 dataset. As shown in~\cref{tab:quant_bits}, our method outperforms uniform quantization and FedPAQ across all quantization bit-widths. Specifically, our method achieves 96.27\% OA with 8-bit quantization, 96.13\% OA with 2-bit quantization, and 96.01\% OA with 1-bit quantization. Our \fedname~outperforms FedPAQ by 0.14\% OA with 8-bit quantization, 0.16\% OA with 2-bit quantization, and 0.22\% OA with 1-bit quantization. The results demonstrate the effectiveness of our proposed quantization method in enhancing the communication efficiency. 

\begin{table}[htbp]
   
    \centering
    \setlength\tabcolsep{9pt}
    \begin{tabular}{lcccc}
    \toprule
    Method & 32-bit & 8-bit & 2-bit & 1-bit \\
    \midrule
    Uniform Quant & 96.33 & 96.01 & 95.87 & 95.56 \\
    FedPAQ~\cite{AISTATS_FedPAQ} & 96.33 & 96.13 & 95.97 & 95.79 \\
    Ours  & 96.33 & \textbf{96.27} & \textbf{96.13} & \textbf{96.01} \\
    \bottomrule
    \end{tabular}
        \caption{\textbf{Quantization bit-width comparison. }}
        \label{tab:quant_bits}
\end{table}

\begin{table}[htbp]
    \centering
    \resizebox{\columnwidth}{!}{%
    \setlength\tabcolsep{4pt}
    \renewcommand\arraystretch{1.6}
    \begin{tabular}{ccc|cccc}\hline \thickhline
    SST & SCG & CSG & RESISC-45 & LoveDA & DIOR-R & LEVIR-CD+ \\ \hline
    \checkmark   &    &    & 94.52   & 51.87     & 64.83    &   72.43     \\
    \checkmark   & \checkmark   &    & 94.70   & 51.85     & 64.91    &   72.55     \\
    \checkmark   &    & \checkmark   & 95.25   & 51.97     & 65.12    &   72.78     \\
    \checkmark  & \checkmark  & \checkmark   & \textbf{96.33}   & \textbf{52.74}     & \textbf{65.51}    &   \textbf{73.21}    \\ \hline \thickhline  
    \end{tabular}%
    }
    \caption{\textbf{Effectiveness of proposed components.}}
    \label{tab:components}
    \end{table}

\textbf{Effectiveness of Components.}
Our ablation study quantitatively validates the complementary nature of proposed components, as summarized in~\cref{tab:components}. The standalone SST mechanism achieves a 94.52\% OA on RESISC-45, 51.87\% mIoU on LoveDA, 64.83\% mAP on DIOR-R, and 72.43\% F1 on LEVIR-CD+. The SCG mechanism further boosts performance across tasks, with 0.18\% OA, 0.02\% mIoU, 0.08\% mAP, and 0.12\% F1 improvements. The CSG mechanism contributes to a 0.83\% OA gain on RESISC-45, 0.29\% mAP on DIOR-R, and 0.35\% F1 on LEVIR-CD+. The full model with all components achieves the best performance, demonstrating the effectiveness of our proposed components in enhancing federated learning for remote sensing tasks.

\begin{table}[htbp]
    \centering
    \resizebox{\columnwidth}{!}{%
    \setlength\tabcolsep{4pt}
    \renewcommand\arraystretch{1.3}
    \begin{tabular}{cc|cccc}
    \hline \thickhline
    \multirow{2}{*}{LE} & \multirow{2}{*}{Round} & RESISC-45 & LoveDA & DIOR-R & LEVIR-CD+ \\ \cline{3-6}
        &              & OA & mIoU & mAP & F1 \\
        \hline
    1   & 100          & \textbf{96.33}   & \textbf{52.74}     & \textbf{65.51}    &   \textbf{73.21} \\
    2   & 50           & 95.82   & 51.93     & 64.08    &   72.05 \\
    4   & 25           & 95.17   & 50.62     & 62.75    &   70.89 \\
    \rowcolor{mygray} 
    100 & 1\textsuperscript{\textdagger} & 94.50   & 48.31     & 60.13    &   68.24 \\ \hline \thickhline
    \end{tabular}%
    }
    \caption{\textbf{Local epochs (LE) impact analysis.} \textdagger~means one-shot FL setting.}
    \label{tab:local_epochs}
    \end{table}

\textbf{Parameter Analysis.}
The trade-off between local computation and communication frequency is systematically investigated through varying local epochs (LE), as shown in~\cref{tab:local_epochs}. Here more local epochs indicate more discrepancy between local and global models, leading to potential client drift. It is worth noting that the one-shot setting (LE=100) suffers catastrophic performance collapse, confirming remote sensing data's inherent heterogeneity demands periodic model synchronization. However, reducing communication rounds and developing one-shot FL is becoming widely adopted for only one round of communication. We expect that our proposed \fedname~can be further improved by incorporating more advanced techniques to handle the one-shot setting.

\section{Conclusion \& Future Work }
\label{conclusion}

This paper takes the first step towards developing a privacy-preserved pre-training framework \textbf{(\fedname)} for RSFMs. \fedname~enables multiple institutions to collaboratively train RSFMs without sharing private data. We introduce Federated Mutual-guidance Learning, which breaks the vicious cycle caused by remote sensing data heterogeneity and high communication overhead. Specifically, we propose a SCG mechanism to guide clients updates towards global-flatness optimal solutions. Additionally, we propose a CSG mechanism to inject local knowledge into the server by low-bit communication. Extensive experiments on four downstream tasks demonstrate the effectiveness of our \fedname~in both full-precision and communication-reduced scenarios, showcasing remarkable communication efficiency and performance gains. In the future, we plan to extend the current framework to support the collaborative pre-training of multi-modal RSFMs with modality heterogeneity.

{
    \small
    \bibliographystyle{ieeenat_fullname}
    \bibliography{main}

\begin{thebibliography}{49}
\providecommand{\natexlab}[1]{#1}
\providecommand{\url}[1]{\texttt{#1}}
\expandafter\ifx\csname urlstyle\endcsname\relax
  \providecommand{\doi}[1]{doi: #1}\else
  \providecommand{\doi}{doi: \begingroup \urlstyle{rm}\Url}\fi

\bibitem[Astruc et~al.(2024)Astruc, Gonthier, Mallet, and Landrieu]{anysat}
Guillaume Astruc, Nicolas Gonthier, Clement Mallet, and Loic Landrieu.
\newblock Anysat: An earth observation model for any resolutions, scales, and modalities, 2024.

\bibitem[Bercea et~al.(2022)Bercea, Wiestler, Rueckert, and Albarqouni]{b}
Cosmin~I Bercea, Benedikt Wiestler, Daniel Rueckert, and Shadi Albarqouni.
\newblock Federated disentangled representation learning for unsupervised brain anomaly detection.
\newblock \emph{Nature Machine Intelligence}, 4\penalty0 (8):\penalty0 685--695, 2022.

\bibitem[Büyüktas et~al.(2024)Büyüktas, Sumbul, and Demir]{GRSM_FL}
Baris Büyüktas, Gencer Sumbul, and Begüm Demir.
\newblock Federated learning across decentralized and unshared archives for remote sensing image classification: A review.
\newblock \emph{IEEE Geoscience and Remote Sensing Magazine}, pages 2--18, 2024.

\bibitem[Caron et~al.(2021)Caron, Touvron, Misra, J\'egou, Mairal, Bojanowski, and Joulin]{dino}
Mathilde Caron, Hugo Touvron, Ishan Misra, Herv\'e J\'egou, Julien Mairal, Piotr Bojanowski, and Armand Joulin.
\newblock Emerging properties in self-supervised vision transformers.
\newblock In \emph{ICCV}, pages 9650--9660, 2021.

\bibitem[Chen and Shi(2020)]{LEVIRCD}
Hao Chen and Zhenwei Shi.
\newblock A spatial-temporal attention-based method and a new dataset for remote sensing image change detection.
\newblock \emph{Remote Sensing}, 12\penalty0 (10), 2020.

\bibitem[Cheng et~al.(2017)Cheng, Han, and Lu]{NWPU}
Gong Cheng, Junwei Han, and Xiaoqiang Lu.
\newblock Remote sensing image scene classification: Benchmark and state of the art.
\newblock \emph{Proceedings of the IEEE}, 105\penalty0 (10):\penalty0 1865--1883, 2017.

\bibitem[Cheng et~al.(2022)Cheng, Wang, Li, Xie, Lang, Yao, and Han]{DIORR}
Gong Cheng, Jiabao Wang, Ke Li, Xingxing Xie, Chunbo Lang, Yanqing Yao, and Junwei Han.
\newblock Anchor-free oriented proposal generator for object detection.
\newblock \emph{IEEE Transactions on Geoscience and Remote Sensing}, 60:\penalty0 1--11, 2022.

\bibitem[Cheng et~al.(2024)Cheng, Zheng, Wu, Qi, and He]{huayi}
Hongquan Cheng, Jie Zheng, Huayi Wu, Kunlun Qi, and Lihua He.
\newblock A communication-efficient distributed deep learning remote sensing image change detection framework.
\newblock \emph{International Journal of Applied Earth Observation and Geoinformation}, 129:\penalty0 103840, 2024.

\bibitem[Gui et~al.(2024)Gui, Chen, Zhang, Cao, Sun, Luo, and Tao]{ssl_TPAMI}
Jie Gui, Tuo Chen, Jing Zhang, Qiong Cao, Zhenan Sun, Hao Luo, and Dacheng Tao.
\newblock A survey on self-supervised learning: Algorithms, applications, and future trends.
\newblock \emph{IEEE Transactions on Pattern Analysis and Machine Intelligence}, 46\penalty0 (12):\penalty0 9052--9071, 2024.

\bibitem[Guo et~al.(2024)Guo, Lao, Dang, Zhang, Yu, Ru, Zhong, Huang, Wu, Hu, He, Wang, Chen, Yang, Zhang, and Li]{skysense}
Xin Guo, Jiangwei Lao, Bo Dang, Yingying Zhang, Lei Yu, Lixiang Ru, Liheng Zhong, Ziyuan Huang, Kang Wu, Dingxiang Hu, Huimei He, Jian Wang, Jingdong Chen, Ming Yang, Yongjun Zhang, and Yansheng Li.
\newblock Skysense: A multi-modal remote sensing foundation model towards universal interpretation for earth observation imagery.
\newblock In \emph{CVPR}, pages 27672--27683, 2024.

\bibitem[Hong et~al.(2024)Hong, Zhang, Li, Li, Li, Yao, Yokoya, Li, Ghamisi, Jia, Plaza, Gamba, Benediktsson, and Chanussot]{SpectralGPT}
Danfeng Hong, Bing Zhang, Xuyang Li, Yuxuan Li, Chenyu Li, Jing Yao, Naoto Yokoya, Hao Li, Pedram Ghamisi, Xiuping Jia, Antonio Plaza, Paolo Gamba, Jon~Atli Benediktsson, and Jocelyn Chanussot.
\newblock Spectralgpt: Spectral remote sensing foundation model.
\newblock \emph{IEEE Transactions on Pattern Analysis and Machine Intelligence}, 46\penalty0 (8):\penalty0 5227--5244, 2024.

\bibitem[Huang et~al.(2024)Huang, Ye, Shi, Wan, Li, Du, and Yang]{FL_survey}
Wenke Huang, Mang Ye, Zekun Shi, Guancheng Wan, He Li, Bo Du, and Qiang Yang.
\newblock Federated learning for generalization, robustness, fairness: A survey and benchmark.
\newblock \emph{IEEE Transactions on Pattern Analysis and Machine Intelligence}, 46\penalty0 (12):\penalty0 9387--9406, 2024.

\bibitem[Iacob et~al.(2024)Iacob, Sani, Marino, Aleksandrov, and Lane]{2}
Alex Iacob, Lorenzo Sani, Bill Marino, Preslav Aleksandrov, and Nicholas~Donald Lane.
\newblock Worldwide federated training of language models.
\newblock \emph{arXiv preprint arXiv:2405.14446}, 2024.

\bibitem[Jaghouar et~al.(2024)Jaghouar, Ong, Basra, Obeid, Straube, Keiblinger, Bakouch, Atkins, Panahi, et~al.]{intellect}
Sami Jaghouar, Jack~Min Ong, Manveer Basra, Fares Obeid, Jannik Straube, Michael Keiblinger, Elie Bakouch, Lucas Atkins, Maziyar Panahi, et~al.
\newblock Intellect-1 technical report.
\newblock \emph{arXiv preprint arXiv:2412.01152}, 2024.

\bibitem[Jing et~al.(2024)Jing, Yu, Zhang, and Zhang]{c}
Shusen Jing, Anlan Yu, Shuai Zhang, and Songyang Zhang.
\newblock Fed{SC}: Provable federated self-supervised learning with spectral contrastive objective over non-i.i.d. data.
\newblock In \emph{ICML}, 2024.

\bibitem[Kim et~al.(2023)Kim, Kwak, Jung, Shin, Kim, and Kim]{w}
Hansol Kim, Youngjun Kwak, Minyoung Jung, Jinho Shin, Youngsung Kim, and Changick Kim.
\newblock Protofl: Unsupervised federated learning via prototypical distillation.
\newblock In \emph{ICCV}, pages 6470--6479, 2023.

\bibitem[Lee et~al.(2024)Lee, Jeong, Kim, Oh, and Yun]{fedsol}
Gihun Lee, Minchan Jeong, Sangmook Kim, Jaehoon Oh, and Se-Young Yun.
\newblock Fedsol: Stabilized orthogonal learning with proximal restrictions in federated learning.
\newblock In \emph{CVPR}, pages 12512--12522, 2024.

\bibitem[Li et~al.(2023)Li, Lyu, Iso, Chakrabarti, and Spranger]{mocosfl}
Jingtao Li, Lingjuan Lyu, Daisuke Iso, Chaitali Chakrabarti, and Michael Spranger.
\newblock Moco{SFL}: enabling cross-client collaborative self-supervised learning.
\newblock In \emph{ICLR}, 2023.

\bibitem[Li et~al.(2024{\natexlab{a}})Li, Zhang, Wang, Liu, Wu, Wang, Zhuang, Xiong, and Yu]{FedMKD}
Mingyi Li, Xiao Zhang, Qi Wang, Tengfei Liu, Ruofan Wu, Weiqiang Wang, Fuzhen Zhuang, Hui Xiong, and Dongxiao Yu.
\newblock Resource-aware federated self-supervised learning with global class representations.
\newblock In \emph{NeurIPS}, 2024{\natexlab{a}}.

\bibitem[Li et~al.(2024{\natexlab{b}})Li, Ye, Fang, Zhao, Chan, Ngai, and Voigt]{3}
Shenghui Li, Fanghua Ye, Meng Fang, Jiaxu Zhao, Yun-Hin Chan, Edith C-H Ngai, and Thiemo Voigt.
\newblock Synergizing foundation models and federated learning: A survey.
\newblock \emph{arXiv preprint arXiv:2406.12844}, 2024{\natexlab{b}}.

\bibitem[Li et~al.(2024{\natexlab{c}})Li, Hong, and Chanussot]{S2MAE}
Xuyang Li, Danfeng Hong, and Jocelyn Chanussot.
\newblock S2mae: A spatial-spectral pretraining foundation model for spectral remote sensing data.
\newblock In \emph{CVPR}, pages 24088--24097, 2024{\natexlab{c}}.

\bibitem[Li et~al.(2025)Li, Tan, Dang, Ye, Bartalev, Shinkarenko, Wang, Zhang, Ru, Guo, Yuan, Yu, Chen, Yang, Junior, and Zhang]{Innovation}
Yansheng Li, Jieyi Tan, Bo Dang, Mang Ye, Sergey~A. Bartalev, Stanislav Shinkarenko, Linlin Wang, Yingying Zhang, Lixiang Ru, Xin Guo, Liangqi Yuan, Lei Yu, Jingdong Chen, Ming Yang, José~Marcato Junior, and Yongjun Zhang.
\newblock Unleashing the potential of remote sensing foundation models via bridging data and computility islands.
\newblock \emph{The Innovation}, page 100841, 2025.

\bibitem[Liao et~al.(2024)Liao, Liu, Chen, Zhou, Yu, Zhu, Yao, Wang, Zheng, and Tan]{FedU2}
Xinting Liao, Weiming Liu, Chaochao Chen, Pengyang Zhou, Fengyuan Yu, Huabin Zhu, Binhui Yao, Tao Wang, Xiaolin Zheng, and Yanchao Tan.
\newblock Rethinking the representation in federated unsupervised learning with non-iid data.
\newblock In \emph{CVPR}, pages 22841--22850, 2024.

\bibitem[Liu et~al.(2021)Liu, Lin, Cao, Hu, Wei, Zhang, Lin, and Guo]{swin}
Ze Liu, Yutong Lin, Yue Cao, Han Hu, Yixuan Wei, Zheng Zhang, Stephen Lin, and Baining Guo.
\newblock Swin transformer: Hierarchical vision transformer using shifted windows.
\newblock In \emph{ICCV}, pages 10012--10022, 2021.

\bibitem[Lu et~al.(2025{\natexlab{a}})Lu, Chen, Chen, Li, Sun, Zheng, Zou, Liang, Li, Jin, Cui, Long, and Feng]{nc_medical}
Senliang Lu, Yehang Chen, Yuan Chen, Peijun Li, Junqi Sun, Changye Zheng, Yujian Zou, Bo Liang, Mingwei Li, Qinggeng Jin, Enming Cui, Wansheng Long, and Bao Feng.
\newblock General lightweight framework for vision foundation model supporting multi-task and multi-center medical image analysis.
\newblock \emph{Nature Communications}, 16\penalty0 (1):\penalty0 2097, 2025{\natexlab{a}}.

\bibitem[Lu et~al.(2025{\natexlab{b}})Lu, Guo, Zimmer-Dauphinee, Nieusma, Wang, VanValkenburgh, Wernke, and Huo]{VFM_RRS}
Siqi Lu, Junlin Guo, James~R Zimmer-Dauphinee, Jordan~M Nieusma, Xiao Wang, Parker VanValkenburgh, Steven~A Wernke, and Yuankai Huo.
\newblock Vision foundation models in remote sensing: A survey.
\newblock \emph{arXiv preprint arXiv:2408.03464}, 2025{\natexlab{b}}.

\bibitem[Lubana et~al.(2022)Lubana, Tang, Kawsar, Dick, and Mathur]{q}
Ekdeep Lubana, Chi~Ian Tang, Fahim Kawsar, Robert Dick, and Akhil Mathur.
\newblock Orchestra: Unsupervised federated learning via globally consistent clustering.
\newblock In \emph{ICML}, pages 14461--14484, 2022.

\bibitem[Maggiori et~al.(2017)Maggiori, Tarabalka, Charpiat, and Alliez]{Inria}
Emmanuel Maggiori, Yuliya Tarabalka, Guillaume Charpiat, and Pierre Alliez.
\newblock Can semantic labeling methods generalize to any city? the inria aerial image labeling benchmark.
\newblock In \emph{IEEE International Geoscience and Remote Sensing Symposium (IGARSS)}. IEEE, 2017.

\bibitem[McMahan et~al.(2017)McMahan, Moore, Ramage, Hampson, and Arcas]{fedavg}
Brendan McMahan, Eider Moore, Daniel Ramage, Seth Hampson, and Blaise Aguera~y Arcas.
\newblock {Communication-efficient learning of deep networks from decentralized data}.
\newblock In \emph{Proceedings of the 20th International Conference on Artificial Intelligence and Statistics}, pages 1273--1282, 2017.

\bibitem[Rehman et~al.(2022)Rehman, Gao, Shen, de~Gusm{\~a}o, and Lane]{fedvssl}
Yasar Abbas~Ur Rehman, Yan Gao, Jiajun Shen, Pedro Porto~Buarque de Gusm{\~a}o, and Nicholas Lane.
\newblock Federated self-supervised learning for video understanding.
\newblock In \emph{ECCV}, pages 506--522, 2022.

\bibitem[Rehman et~al.(2023)Rehman, Gao, de~Gusmao, Alibeigi, Shen, and Lane]{L-DAWA}
Yasar Abbas~Ur Rehman, Yan Gao, Pedro Porto~Buarque de Gusmao, Mina Alibeigi, Jiajun Shen, and Nicholas~D. Lane.
\newblock L-dawa: Layer-wise divergence aware weight aggregation in federated self-supervised visual representation learning.
\newblock In \emph{ICCV}, pages 16464--16473, 2023.

\bibitem[Reisizadeh et~al.(2020)Reisizadeh, Mokhtari, Hassani, Jadbabaie, and Pedarsani]{AISTATS_FedPAQ}
Amirhossein Reisizadeh, Aryan Mokhtari, Hamed Hassani, Ali Jadbabaie, and Ramtin Pedarsani.
\newblock Fedpaq: A communication-efficient federated learning method with periodic averaging and quantization.
\newblock In \emph{Proceedings of the Twenty Third International Conference on Artificial Intelligence and Statistics}, pages 2021--2031, 2020.

\bibitem[Sani et~al.(2024{\natexlab{a}})Sani, Iacob, Cao, Lee, Marino, Gao, Cai, Li, Zhao, Qiu, and Lane]{FedLLM}
Lorenzo Sani, Alex Iacob, Zeyu Cao, Royson Lee, Bill Marino, Yan Gao, Dongqi Cai, Zexi Li, Wanru Zhao, Xinchi Qiu, and Nicholas~D. Lane.
\newblock Photon: Federated llm pre-training.
\newblock \emph{arXiv preprint arXiv:2411.02908}, 2024{\natexlab{a}}.

\bibitem[Sani et~al.(2024{\natexlab{b}})Sani, Iacob, Cao, Marino, Gao, Paulik, Zhao, Shen, Aleksandrov, Qiu, et~al.]{1}
Lorenzo Sani, Alex Iacob, Zeyu Cao, Bill Marino, Yan Gao, Tomas Paulik, Wanru Zhao, William~F Shen, Preslav Aleksandrov, Xinchi Qiu, et~al.
\newblock The future of large language model pre-training is federated.
\newblock \emph{arXiv preprint arXiv:2405.10853}, 2024{\natexlab{b}}.

\bibitem[Sun et~al.(2020)Sun, Chen, Giannakis, Yang, and Yang]{TPAMI_LAQ}
Jun Sun, Tianyi Chen, Georgios~B Giannakis, Qinmin Yang, and Zaiyue Yang.
\newblock Lazily aggregated quantized gradient innovation for communication-efficient federated learning.
\newblock \emph{IEEE Transactions on Pattern Analysis and Machine Intelligence}, 44\penalty0 (4):\penalty0 2031--2044, 2020.

\bibitem[Sun et~al.(2023)Sun, Wang, Lu, Zhu, Lu, He, Li, Rong, Yang, Chang, He, Yang, Wang, Lu, and Fu]{ringmo}
Xian Sun, Peijin Wang, Wanxuan Lu, Zicong Zhu, Xiaonan Lu, Qibin He, Junxi Li, Xuee Rong, Zhujun Yang, Hao Chang, Qinglin He, Guang Yang, Ruiping Wang, Jiwen Lu, and Kun Fu.
\newblock Ringmo: A remote sensing foundation model with masked image modeling.
\newblock \emph{IEEE Transactions on Geoscience and Remote Sensing}, 61:\penalty0 1--22, 2023.

\bibitem[Tolan et~al.(2024)Tolan, Yang, Nosarzewski, Couairon, Vo, Brandt, Spore, Majumdar, Haziza, Vamaraju, et~al.]{rse_dinov2}
Jamie Tolan, Hung-I Yang, Benjamin Nosarzewski, Guillaume Couairon, Huy~V Vo, John Brandt, Justine Spore, Sayantan Majumdar, Daniel Haziza, Janaki Vamaraju, et~al.
\newblock Very high resolution canopy height maps from rgb imagery using self-supervised vision transformer and convolutional decoder trained on aerial lidar.
\newblock \emph{Remote Sensing of Environment}, 300:\penalty0 113888, 2024.

\bibitem[Tuia et~al.(2024)Tuia, Schindler, Demir, Zhu, Kochupillai, Džeroski, van Rijn, Hoos, Del~Frate, Datcu, Markl, Le~Saux, Schneider, and Camps-Valls]{GRSM_AI}
Devis Tuia, Konrad Schindler, Begüm Demir, Xiao~Xiang Zhu, Mrinalini Kochupillai, Sašo Džeroski, Jan~N. van Rijn, Holger~H. Hoos, Fabio Del~Frate, Mihai Datcu, Volker Markl, Bertrand Le~Saux, Rochelle Schneider, and Gustau Camps-Valls.
\newblock Artificial intelligence to advance earth observation: A review of models, recent trends, and pathways forward.
\newblock \emph{IEEE Geoscience and Remote Sensing Magazine}, pages 2--25, 2024.

\bibitem[Tun et~al.(2024)Tun, Thwal, Huy, Nguyen, and Hong]{e}
Ye~Lin Tun, Chu~Myaet Thwal, Le~Quang Huy, Minh~NH Nguyen, and Choong~Seon Hong.
\newblock Lw-fedssl: Resource-efficient layer-wise federated self-supervised learning.
\newblock \emph{arXiv preprint arXiv:2401.11647}, 2024.

\bibitem[Wang et~al.(2021)Wang, Zheng, Ma, Lu, and Zhong]{LoveDA}
Junjue Wang, Zhuo Zheng, Ailong Ma, Xiaoyan Lu, and Yanfei Zhong.
\newblock Love{DA}: A remote sensing land-cover dataset for domain adaptive semantic segmentation.
\newblock In \emph{NeurIPS}, 2021.

\bibitem[Wang et~al.(2023)Wang, Zhang, Li, Tian, and Tedrake]{a}
Lirui Wang, Kaiqing Zhang, Yunzhu Li, Yonglong Tian, and Russ Tedrake.
\newblock Does learning from decentralized non-{IID} unlabeled data benefit from self supervision?
\newblock In \emph{ICLR}, 2023.

\bibitem[Xia et~al.(2017)Xia, Hu, Hu, Shi, Bai, Zhong, Zhang, and Lu]{AID}
Gui-Song Xia, Jingwen Hu, Fan Hu, Baoguang Shi, Xiang Bai, Yanfei Zhong, Liangpei Zhang, and Xiaoqiang Lu.
\newblock Aid: A benchmark data set for performance evaluation of aerial scene classification.
\newblock \emph{IEEE Transactions on Geoscience and Remote Sensing}, 55\penalty0 (7):\penalty0 3965--3981, 2017.

\bibitem[Xia et~al.(2018)Xia, Bai, Ding, Zhu, Belongie, Luo, Datcu, Pelillo, and Zhang]{DOTA}
Gui-Song Xia, Xiang Bai, Jian Ding, Zhen Zhu, Serge Belongie, Jiebo Luo, Mihai Datcu, Marcello Pelillo, and Liangpei Zhang.
\newblock Dota: A large-scale dataset for object detection in aerial images.
\newblock In \emph{CVPR}, 2018.

\bibitem[Xie et~al.(2024)Xie, Lin, Toh, and Zhou]{loco}
Xingyu Xie, Zhijie Lin, Kim-Chuan Toh, and Pan Zhou.
\newblock Loco: Low-bit communication adaptor for large-scale model training.
\newblock \emph{arXiv preprint arXiv:2407.04480}, 2024.

\bibitem[Xie et~al.(2022)Xie, Zhang, Cao, Lin, Bao, Yao, Dai, and Hu]{simmim}
Zhenda Xie, Zheng Zhang, Yue Cao, Yutong Lin, Jianmin Bao, Zhuliang Yao, Qi Dai, and Han Hu.
\newblock Simmim: A simple framework for masked image modeling.
\newblock In \emph{CVPR}, pages 9653--9663, 2022.

\bibitem[Yan et~al.(2023)Yan, Qu, Wei, Huang, Shen, Rubin, Xing, and Zhou]{SSL-FL}
Rui Yan, Liangqiong Qu, Qingyue Wei, Shih-Cheng Huang, Liyue Shen, Daniel~L. Rubin, Lei Xing, and Yuyin Zhou.
\newblock Label-efficient self-supervised federated learning for tackling data heterogeneity in medical imaging.
\newblock \emph{IEEE Transactions on Medical Imaging}, 42\penalty0 (7):\penalty0 1932--1943, 2023.

\bibitem[Yang et~al.(2020)Yang, Xia, Liu, Du, Yang, Pelillo, and Zhang]{SECOND}
Kunping Yang, Gui-Song Xia, Zicheng Liu, Bo Du, Wen Yang, Marcello Pelillo, and Liangpei Zhang.
\newblock Semantic change detection with asymmetric siamese networks.
\newblock \emph{arXiv preprint arXiv:2010.05687}, 2020.

\bibitem[Zhang et~al.(2024)Zhang, Yang, Hu, Gong, and Zhang]{rsfm}
Mi Zhang, Bingnan Yang, Xiangyun Hu, Jianya Gong, and Zuxun Zhang.
\newblock Foundation model for generalist remote sensing intelligence: Potentials and prospects.
\newblock \emph{Science Bulletin}, 69\penalty0 (23):\penalty0 3652--3656, 2024.

\bibitem[Zhuang et~al.(2022)Zhuang, Wen, and Zhang]{fedema}
Weiming Zhuang, Yonggang Wen, and Shuai Zhang.
\newblock Divergence-aware federated self-supervised learning.
\newblock In \emph{ICLR}, 2022.

\end{thebibliography}
}

 \clearpage
\setcounter{page}{1}
\setcounter{section}{0}
\maketitlesupplementary
\renewcommand\thesection{\Alph{section}}

\section{Overview}
We provide the following materials to supplement our paper and divide them into two sections.
\begin{itemize}
    \item We provide the theoretical analysis of our proposed FedSense in~\cref{sec:theory}.
    \item We provide the details of our pre-training datasets and downstream datasets in~\cref{sec:datasets}
\end{itemize}

\section{Theoretical Analysis}
\label{sec:theory}

\subsection{Assumptions}
\begin{assumption}[Smoothness]\label{assump:smooth}
The self-supervised loss $\mathcal{L}_m^{\text{ssl}}$ is $L$-smooth: 
\begin{equation}
\|\nabla\mathcal{L}_m^{\text{ssl}}(\theta_1) - \nabla\mathcal{L}_m^{\text{ssl}}(\theta_2)\| \leq L\|\theta_1 - \theta_2\|,\quad \forall \theta_1,\theta_2
\end{equation}
\end{assumption}
 
\begin{assumption}[Bounded Gradient]\label{assump:grad}
Local gradients are bounded: 
\begin{equation}
\mathbb{E}[\|\nabla\mathcal{L}_m^{\text{total}}(\theta_m)\|^2] \leq G^2,\quad \forall m
\end{equation}
\end{assumption}
 
\begin{assumption}[Parameter Discrepancy]\label{assump:disc}
The discrepancy between local and global models satisfies:
\begin{equation}
\|\theta_m - \Theta\| \leq \delta,\quad \forall m \in[M]
\end{equation}
where $\delta$ quantifies the maximum client drift.
\end{assumption}
 
\subsection{Key Lemmas}
 
\begin{lemma}[Optimal Perturbation Bound]\label{lemma:perturb}
Under Assumption \ref{assump:grad}, the optimal perturbation $\widetilde{\epsilon}$ in SCG satisfies:
\begin{equation}
\|\widetilde{\epsilon}\| \leq \lambda\sqrt{\beta^2\delta^2 + G^2}
\end{equation}
\end{lemma}
 
\begin{proof}
From the perturbation approximation:
\begin{align*}
\widetilde{\epsilon} &\approx \lambda\frac{\nabla\mathcal{L}_m^{\text{disc}}}{\|\nabla\mathcal{L}_m^{\text{disc}}\|} \\
\|\widetilde{\epsilon}\| &\leq \lambda\sqrt{\frac{\|\nabla\mathcal{L}_m^{\text{disc}}\|^2}{\|\nabla\mathcal{L}_m^{\text{disc}}\|^2}} = \lambda
\end{align*}
Using the parameter discrepancy term $\nabla\mathcal{L}_m^{\text{disc}} = \beta(\theta_m - \Theta)$ and Assumption \ref{assump:disc}:
\begin{equation*}
\|\nabla\mathcal{L}_m^{\text{disc}}\| \leq \beta\delta
\end{equation*}
Combining with gradient bound $G$ via triangle inequality completes the proof.
\end{proof}
 
\begin{lemma}[Quantization Error Decay]\label{lemma:quant}
Let $e_m^t$ be the feedback error in CSG. With momentum factor $\alpha \in(0,1)$, the error decays geometrically:
\begin{equation}
\|e_m^t\| \leq \alpha^t\|e_m^0\| + \frac{1-\alpha}{1-\alpha^{t+1}}\sum_{k=0}^t\alpha^{t-k}\|\epsilon_q^k\|
\end{equation}
where $\epsilon_q^k$ is the quantization error at round $k$.
\end{lemma}
 
\begin{proof}
Unrolling the recursive error update:
\begin{align*}
e_m^t &= \alpha e_m^{t-1} + (1-\alpha)\epsilon_q^t \\
&= \alpha^t e_m^0 + (1-\alpha)\sum_{k=1}^t \alpha^{t-k}\epsilon_q^k
\end{align*}
Taking norms and applying triangle inequality:
\begin{align*}
\|e_m^t\| &\leq \alpha^t\|e_m^0\| + (1-\alpha)\sum_{k=1}^t\alpha^{t-k}\|\epsilon_q^k\| \\
&\leq \alpha^t\|e_m^0\| + \frac{1-\alpha}{1-\alpha^{t+1}}\sum_{k=0}^t\alpha^{t-k}\|\epsilon_q^k\|
\end{align*}
The geometric series bound completes the proof.
\end{proof}
 
\subsection{Main Convergence Result}
 
\begin{theorem}[Convergence Guarantee]\label{thm:converge}
Under Assumptions \ref{assump:smooth}-\ref{assump:disc}, let learning rate $\gamma = \frac{1}{L\sqrt{T}}$. After $T$ rounds, the averaged gradient satisfies:
\begin{equation}
\frac{1}{T}\sum_{t=1}^T\mathbb{E}\|\nabla\mathcal{L}^{\text{total}}(\Theta^t)\|^2 \leq \frac{2L(\mathcal{L}^0 - \mathcal{L}^*)}{\sqrt{T}} + \frac{C}{T}\sum_{t=1}^T(\delta^2 + \|e^t\|^2)
\end{equation}
where $C$ is a constant combining $L,G,\beta,\lambda$.
\end{theorem}
 
\begin{proof}[Proof Sketch]
Using smoothness (Assump. \ref{assump:smooth}):
\begin{align*}
\mathcal{L}^{t+1} &\leq \mathcal{L}^t + \langle\nabla\mathcal{L}^t, \Theta^{t+1}-\Theta^t\rangle + \frac{L}{2}\|\Theta^{t+1}-\Theta^t\|^2
\end{align*}
Substituting the update rule $\Theta^{t+1} = \Theta^t - \gamma(\nabla\mathcal{L}^{\text{total}} + e^t)$:
\begin{align*}
\mathbb{E}[\mathcal{L}^{t+1}] &\leq \mathbb{E}[\mathcal{L}^t] - \gamma\mathbb{E}\|\nabla\mathcal{L}^t\|^2 + \gamma\mathbb{E}\langle\nabla\mathcal{L}^t,e^t\rangle \\
&\quad + \frac{L\gamma^2}{2}\mathbb{E}\|\nabla\mathcal{L}^t + e^t\|^2
\end{align*}
\end{proof}

\section{Dataset details and implementation details}
\label{sec:datasets}

\begin{figure*}[htbp]
    \centering
    \includegraphics[width=1\linewidth]{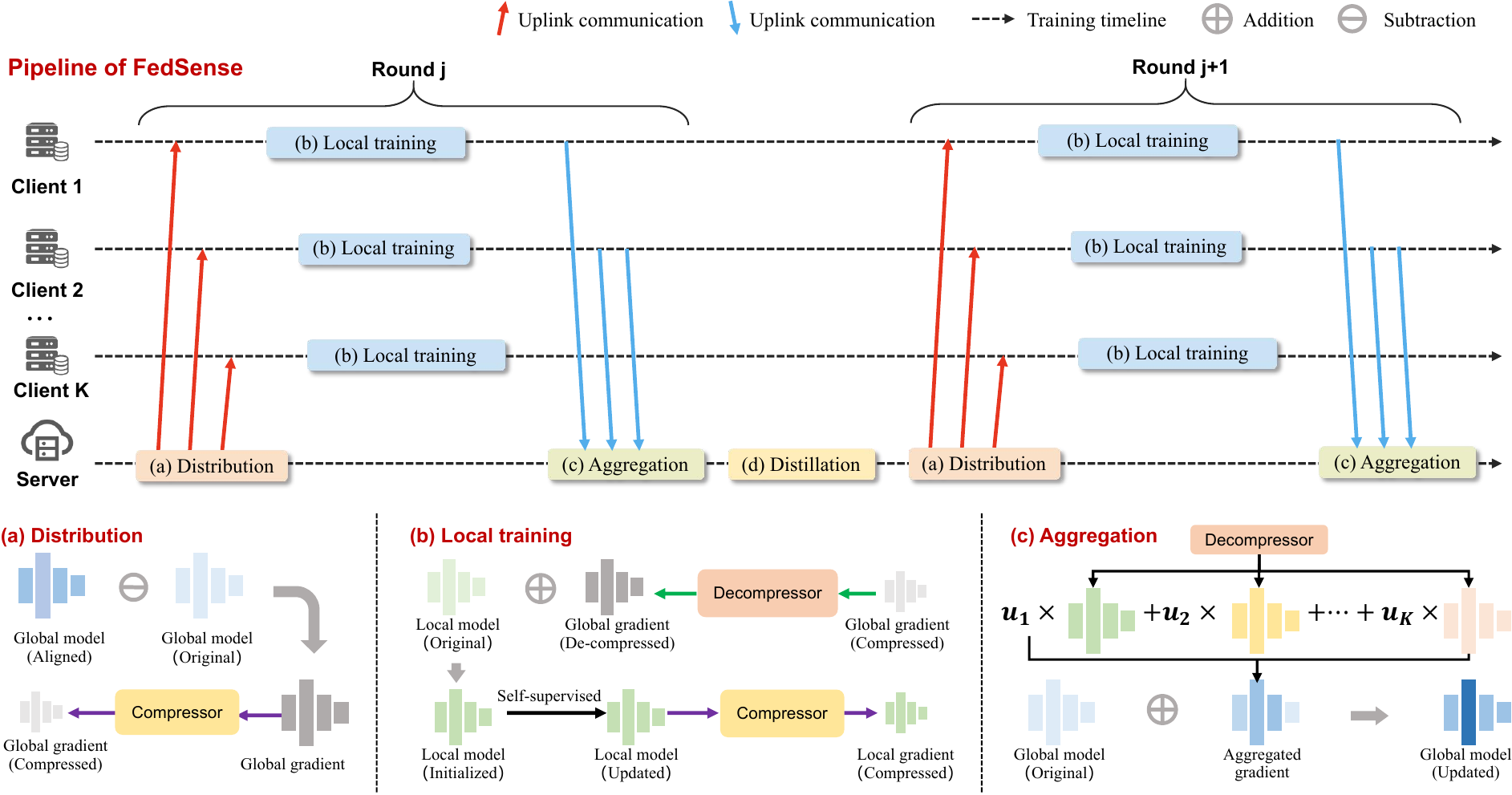}
    \caption{\textbf{Pipeline of privacy-preserved pre-training of RSFMs.}}
    \label{fig:pipeline}
\end{figure*}

\begin{table}[htbp]
    \resizebox{\columnwidth}{!}{%
    \setlength\tabcolsep{4pt}
    \renewcommand\arraystretch{1.6}
    \begin{tabular}{ccrcc}
        \hline \thickhline
    \textbf{ID} & \textbf{Source} & \textbf{\#samples} & \textbf{GSD (m)} & \textbf{Coverage} \\  \hline
    Server      & WorldView-3/4   & 240,000            & 0.5-2.5          & Global            \\
    \hline
    Client 01   & NOAA            & 22,292             & 0.25             & USA               \\
    Client 02   & GF-2            & 27,300             & 4.0              & China             \\
    Client 03   & WorldView-2     & 88,272             & 0.3-0.5          & Global            \\
    Client 04   & Mixed           & 125,474            & 0.3-25           & Global            \\
    Client 05   & QB-2/GE-1       & 180,562            & 0.3              & Global            \\
    Client 06   & JL-1/GF-7       & 40,816             & 0.8              & China             \\
    Client 07   & Mixed           & 90,000             & 0.3-25           & Global            \\
    Client 08   & QB-2/GE-1       & 180,000            & 0.3-25           & Global            \\
    Client 09   & NAIP            & 45,000             & 1.25             & USA               \\
    Client 10   & Mixed           & 50,800             & 0.3-0.75         & Global            \\
    \hline
    Total       & Multi-source    & 1,000,000          & 0.25-25          & Global    
    \\   \hline \thickhline       
    \end{tabular}%
    }
    \caption{\textbf{Details of the pre-training datasets.}}
    \end{table}

\begin{figure}[htbp]
    \centering
    \includegraphics[width=1\linewidth]{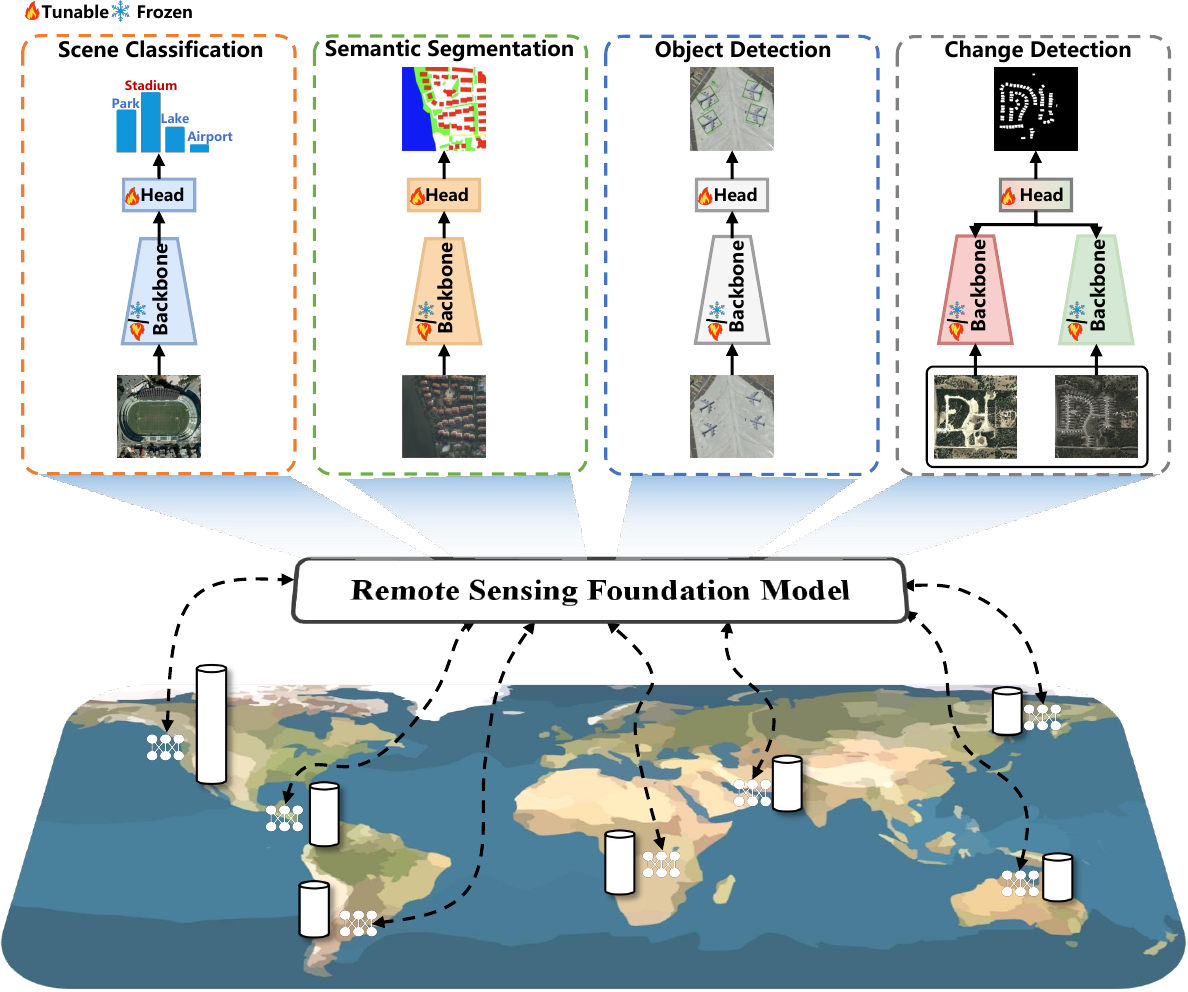}
    \caption{\textbf{Illustration on downstream usage of collaboratively pre-trained RSFMs to accommodate various Earth Observation tasks.}}
    \label{fig:down}
\end{figure}

\textbf{Scene Classification}. 
\begin{enumerate}[(1)]
    \item \textit{Aerial Image Dataset (AID)~\cite{AID}}. This dataset is comprised of 10,000 images across 30 classes, all sourced from Google Earth and cropped to 600$\times$600 pixels. It also includes diverse resolutions from 0.5 to 8 meters per pixel and geographic variations to enhance robustness. 
    \item \textit{NWPU-RESISC45~\cite{NWPU}}. This dataset comprises 31,500 RGB images at resolutions from 0.2m to 30m across 45 scene classes, each with 700 samples with a size of 256 times 256 pixels. It offers significant variability in translation, scale, viewpoint, illumination, and occlusion. It also has high within-class diversity and inter-class similarity.
\end{enumerate}

\textbf{Object Detection}.
\begin{enumerate}[(1)]
    \item \textit{DIOR-R~\cite{DIORR}}. This dataset consists of 23,463 remote sensing images, with 192,472 annotated object instances spanning 20 categories. The size of each image is 800$\times$800 pixels, and spatial resolutions range from 0.5m to 30m. With emphasis on high inter-class similarity, intra-class diversity, and object size variability, it is designed to benchmark object detection methods in diverse conditions such as different imaging times, weathers, and resolutions.
    \item \textit{DOTA-v1.0~\cite{DOTA}}. This dataset consists of 2,806 aerial images, measuring from 800$\times$800 to 4000$\times$4000 pixels, annotated with 188,282 instances across 15 categories that include airplanes, ships, vehicles, and bridges. The objects in this dataset are presented in diverse scales, orientations and aspect ratios.
\end{enumerate}

\textbf{Semantic Segmentation}.

\begin{enumerate}[(1)]
    \item \textit{LoveDA~\cite{LoveDA}}. This dataset for domain-adaptive semantic segmentation features 5,987 images with spatial resolution of 0.3 m, each sized at 1024$\times$1024 pixels in RGB format. Covering 536.15 $km^2$, it spans urban and rural areas across Nanjing, Changzhou and Wuhan, and includes pixel-level annotations across seven land-cover categories. It addresses challenges of multi-scale objects, complex backgrounds, and inconsistent class distributions, supporting semantic segmentation and unsupervised domain adaptation.
    \item \textit{Inria~\cite{Inria}}. This dataset comprises high-resolution RGB aerial imagery, with 180 training and 180 test tiles (each 1500$\times$1500 pixels, 0.3 m resolution). It focuses on two classes: building and non-building, covering a total of 405 $km^2$ of urban areas across five cities in the U.S. and Austria. The dataset emphasizes generalization challenges, supporting semantic segmentation across diverse urban landscapes.
\end{enumerate}

\textbf{Change Detection}. 
\begin{enumerate}[(1)]
    \item \textit{LEVIR-CD+~\cite{LEVIRCD}}. This dataset is a high-resolution urban building change detection dataset comprised of 985 RGB image pairs from Google Earth, each measuring 1024$\times$1024 pixels with a spatial resolution of 0.5 meters per pixel. Spanning 20 regions in Texas, the dataset includes building and land use change masks, covering the years 2002 to 2020 with a 5-year observation interval. It features predominantly urban areas and near-nadir imagery, making it accessible for change detection studies.
    \item \textit{SECOND~\cite{SECOND}}. This dataset is a large-scale semantic change detection benchmark, comprising 4,662 image pairs, each with a size of 512$\times$512 pixels. The images were collected from multiple platforms across multiple cities including Hangzhou, Chengdu, and Shanghai. It focuses on six land-cover classes: non-vegetated ground surface, tree, low vegetation, water, buildings, and playgrounds. SECOND also offers approximately 30 change types, including changes within the same land-cover class, with pixel-level annotations ensuring high diversity and label accuracy.
\end{enumerate}

\end{document}